\documentclass[conference]{IEEEtran}
\usepackage{amsmath,amssymb,amsfonts}
\usepackage{algorithmic}
\usepackage{algorithm}
\usepackage{graphicx}
\usepackage{textcomp}
\usepackage{xcolor}
\usepackage{booktabs}
\usepackage{hyperref}
\usepackage{cite}
\usepackage{url}
\usepackage{listings}
\usepackage{tabularx}
\usepackage{amsthm}
\usepackage{mdframed}

\hypersetup{
  colorlinks=true,
  urlcolor=blue,
  linkcolor=blue,
  citecolor=blue,
  filecolor=blue,
  pdftitle={PRIMA: Operational Patterns for Resilient Multi-Agent Research},
  pdfauthor={Sasank Annapureddy}
}

\newtheorem{theorem}{Theorem}

\begin{document}

\title{PRIMA: Operational Patterns for Resilient Multi-Agent\\Research with Verifiable Identity and Convergent Feedback}

\author{
\IEEEauthorblockN{Sasank Annapureddy}
\IEEEauthorblockA{\href{mailto:sasank11@icloud.com}{sasank11@icloud.com}}
}

\maketitle

\begin{abstract}
Large language models (LLMs) can produce capable single-turn outputs but operating them as a coordinated, multi-agent research system over multi-hour wall-clock runs surfaces a class of failure modes that single-shot evaluation cannot. Upstream providers throttle without warning; orchestrators silently retry through degraded modes; sub-agents drift the task to fit accessible tools, narrate machinery instead of using it, open revision iterations with extensive self-apology, or treat upstream context as executable directives. We present \textbf{PRIMA}, a multi-agent research system whose primary contributions are a set of \emph{operational patterns} for surviving these failure modes in production: a resilience-and-recovery layer that detects upstream rate-limit signals at the provider boundary, persists a typed pause record to disk, and resumes long-running runs without re-executing converged work even across process restarts; a sub-agent operating discipline that encodes task-fidelity, tool-use, revision, and inter-step context-boundary norms as a structural prompt layer; and a multi-phase application pattern for structured engineering deliverables that pairs orthogonal draft steps with an explicit cross-document harmonization pass before final synthesis, addressing the temporal-asymmetry problem in sequentially-generated artifact bundles. These operational contributions sit atop a foundational protocol comprising a research-program specification language with explicit convergence criteria, a dual-metric scoring engine that combines LLM-judged rubric evaluation with sandboxed code execution, an outer meta-optimization loop that evolves the research program across generations, event-driven persistence with tree-structured branching, a hook-based middleware system, context compaction via LLM summarization, and a multi-provider LLM abstraction. Agent identities are derived from prime powers, providing collision-free identifiers and trivially-verifiable cluster membership without a central registry. Theoretical guarantees include $O(k)$ verification time for an agent at depth $k$, $O(V+E)$ DAG validation, and identity collision freedom by the Fundamental Theorem of Arithmetic. A case study---an open-problem investigation of Graph Isomorphism in which the protocol was designed to produce an original algorithmic proposal rather than a literature survey---grounds the architectural claims in a concrete generated artifact.\footnote{Certain implementation details---exact prompt text, specific signal patterns used in resilience detection, and the contractual contents of structured deliverable steps---are withheld pending intellectual-property disclosure; the architectural claims and theoretical guarantees presented here are unaffected. See \S\ref{sec:disclosure}.}
\end{abstract}


\section{Introduction}

The application of large language models to open-ended research tasks reveals a fundamental tension: LLMs produce capable single-turn outputs but lack mechanisms for iterative refinement toward verifiable quality targets. When research requires multiple specialized perspectives---literature review, methodology design, quantitative analysis, synthesis---the challenge compounds: agents must coordinate, share context, and produce outputs whose provenance is traceable and verifiable.

Existing multi-agent frameworks \cite{wu2023, hong2023} address coordination through role-based prompting and centralized message routing. However, these systems suffer from three critical limitations: (1) \textbf{identity opacity}---agent identities are arbitrary strings with no mathematical guarantees against collision or spoofing; (2) \textbf{convergence ambiguity}---there is no formal criterion for when an agent's output is ``good enough''; (3) \textbf{static programs}---the research protocol itself cannot evolve based on execution results.

This paper introduces \textbf{PRIMA}, a system that addresses all three limitations through a novel combination of prime-power identity, dual-metric convergence scoring, and meta-level program optimization.

\subsection{Contributions}

The paper's primary contributions are three operational patterns for long-running multi-agent research systems, supported by foundational infrastructure:

\paragraph{Operational patterns (primary)}
\begin{enumerate}
\item A \emph{resilience and recovery protocol} for long-running multi-agent runs, including typed rate-limit detection at the provider boundary, on-disk pause records that survive process restarts, model-preserving resume semantics, and a strengthened completed-step loader that rejects poisoned outputs.
\item A \emph{sub-agent operating discipline}---a structural prompt layer encoding task-fidelity, tool-use, revision, and inter-step context-boundary norms---that mitigates known multi-agent failure modes such as task drift, machinery narration, apology loops, and indirect prompt injection through context.
\item An \emph{application pattern for structured engineering deliverables} that pairs orthogonal draft phases with an explicit cross-document harmonization pass before final synthesis, addressing the temporal-asymmetry problem in multi-step generation.
\end{enumerate}

\paragraph{Foundational infrastructure (supporting)}
\begin{enumerate}
\setcounter{enumi}{3}
\item A \emph{human-readable research protocol} specifying convergence criteria, dependency DAGs, and metric configurations in a machine-executable markdown format.
\item A \emph{dual-metric scoring engine} that combines LLM-judged rubric evaluation with sandboxed quantitative metric execution, enabling hybrid scoring with configurable weighting.
\item A \emph{meta-optimization loop architecture} that analyzes convergence reports and generates revised program versions across generations; we describe the design and report architectural properties but leave quantitative effectiveness evaluation to future work.
\item A \emph{prime-power identity scheme} providing collision-free, easily-verifiable agent identities by injection from (cluster, depth) into the integers via $p^k$, requiring no central registry beyond the cluster--prime mapping.
\item An \emph{event-driven persistence layer} with tree-structured branching, hook-based middleware, and context compaction.
\item A \emph{multi-provider LLM abstraction} supporting multiple provider backends with automatic discovery and graceful fallback, requiring no API keys when using CLI-based providers.
\end{enumerate}

\section{Background and Motivation}

\subsection{Why Multi-Agent Research Systems Fail}

Current multi-agent LLM systems rely on one of four coordination strategies: (1) \textbf{Role-based prompting} (AutoGen \cite{wu2023}, MetaGPT \cite{hong2023}), which assigns personas without formal identity guarantees; (2) \textbf{Centralized routing}, creating single points of failure and bottlenecks; (3) \textbf{Fixed-threshold evaluation}, applying uniform quality bars regardless of task difficulty; (4) \textbf{Static protocols}, where the research plan cannot adapt to discovered difficulties.

All four conflate \emph{coordination} with \emph{verification}. They orchestrate agents but cannot prove which agent produced which output, whether the output meets domain-specific quantitative criteria, or whether the research protocol itself is optimal.

\subsection{On Identity}

Several of our infrastructure pieces depend on having stable, machine-verifiable agent identifiers without requiring a central registry to disambiguate them at audit time. We adopt prime-power identities ($\text{id} = p^k$) for this purpose. The construction is mathematically simple: any injection from (cluster, depth) into the integers would suffice, but the prime-power form makes the inverse---given an identity, recover (cluster, depth)---a trivial factorization. We do not claim this is the only viable identity scheme or that prime-power identities have cryptographic properties; the relevant primes ($3, 5, \ldots, 997$) are far below cryptographic key sizes. The claim is operational: identities are easy to assign, easy to verify, and the consensus token---the product over all converged identities---uniquely encodes the contributing multiset under factorization.

\section{Prime Identity Framework}

\subsection{Mathematical Foundation}

PRIMA assigns identities using pure prime powers. Let $\mathcal{P} = \{p_1, p_2, \ldots, p_N\}$ be a pool of $N$ distinct primes. Each research cluster (corresponding to one step in the program) is assigned a unique prime $p_i \in \mathcal{P}$. An agent operating at depth $k$ within cluster $p_i$ receives identity:
\begin{equation}
\text{id}(p_i, k) = p_i^k
\end{equation}

\begin{theorem}[Collision Freedom]
For distinct primes $p_i \neq p_j$ and positive integers $k, l$, we have $p_i^k \neq p_j^l$. This follows directly from the Fundamental Theorem of Arithmetic.
\end{theorem}

\begin{proof}
Suppose $p_i^k = p_j^l$ for $p_i \neq p_j$. By the Fundamental Theorem of Arithmetic, every positive integer has a unique prime factorization. The number $p_i^k$ has factorization $\{p_i: k\}$, while $p_j^l$ has factorization $\{p_j: l\}$. Since $p_i \neq p_j$, these factorizations are distinct, contradicting uniqueness. \qed
\end{proof}

\subsection{Verification Protocol}

Given an agent identity $n$, membership verification proceeds via trial division:

\begin{algorithm}[t]
\caption{Agent Identity Verification}
\begin{algorithmic}[1]
\REQUIRE Agent identity $n$, expected cluster prime $p$
\ENSURE Verification result $(valid, depth)$
\STATE $k \leftarrow 0$
\STATE $m \leftarrow n$
\WHILE{$m \bmod p = 0$}
  \STATE $m \leftarrow m / p$
  \STATE $k \leftarrow k + 1$
\ENDWHILE
\IF{$m = 1$ \AND $k > 0$}
  \RETURN $(true, k)$
\ELSE
  \RETURN $(false, 0)$
\ENDIF
\end{algorithmic}
\end{algorithm}

This verification runs in $O(\log_p n) = O(k)$ time, requiring no network calls, no database lookups, and no trust assumptions. Any party holding an agent ID and the expected cluster prime can independently verify membership.

\subsection{Capacity and Bounds}

The implementation draws cluster primes from a bounded prime pool. For 64-bit identity compatibility, the maximum depth for prime $p$ is:
\begin{equation}
k_{\max}(p) = \lfloor \log_p(2^{63} - 1) \rfloor
\end{equation}

Total capacity in a 64-bit identity space comfortably exceeds the practical requirements of a single research program (thousands of unique agent identities across hundreds of clusters), and the pool can be enlarged trivially by extending the prime sieve.

\subsection{Consensus Token}

Upon completion, the system computes a consensus token as the product of all converged agent identities:
\begin{equation}
\mathcal{T} = \prod_{i \in \text{converged}} \text{id}_i
\end{equation}

By the Fundamental Theorem of Arithmetic, $\mathcal{T}$ uniquely encodes which agents converged and at what depths. Any verifier can factorize $\mathcal{T}$ to reconstruct the exact set of contributing agents without consulting any registry.

\section{Research Protocol: \texttt{program.md}}

\subsection{Design Principles}

PRIMA executes research programs specified in \texttt{program.md}---a markdown document combining YAML configuration blocks with human-readable descriptions. The format embodies three principles: (1) \emph{human-readable}: researchers can author and review programs without tooling; (2) \emph{machine-executable}: the parser produces a fully typed \texttt{ResearchProgram} dataclass; (3) \emph{version-trackable}: frontmatter includes version, parent version, generation score, and mutation log for meta-optimization lineage.

\subsection{Structure}

A \texttt{program.md} contains six sections:

\begin{table}[h]
\caption{Structure of a \texttt{program.md} research protocol.}
\label{tab:program}
\centering
\small
\begin{tabular}{@{}lp{5.4cm}@{}}
\toprule
\textbf{Section} & \textbf{Contents} \\
\midrule
Frontmatter & Title, author, version, parent\_version, generation\_score, mutation\_log, tags, area, cluster\_prime \\
\addlinespace
Problem & Free-text problem statement providing research context \\
\addlinespace
Convergence & global\_threshold ($\in [0,1]$), max\_iterations, scoring\_method, metric\_weight ($\alpha$), feedback\_style, timeout, early\_stop, retry\_on\_failure \\
\addlinespace
Steps & Per-step: id, name, depth, goal, accept\_criteria (list), metric config (optional), depends\_on, context\_from, output format \\
\addlinespace
Dependencies & Directed graph, parallel\_groups for phased execution, execution\_mode (phased/sequential/eager) \\
\addlinespace
Output & Format, include\_audit\_trail, export\_formats \\
\bottomrule
\end{tabular}
\end{table}

\subsection{Step Configuration}

Each step defines the scope of work for one agent cluster. The critical configuration elements are:

\textbf{Acceptance Criteria.} A list of human-readable criteria against which the LLM evaluator judges output quality. Example: ``Identifies at least three independent approaches to the problem.''

\textbf{Metric Configuration.} An optional quantitative evaluation block specifying:
\begin{itemize}
\item \texttt{type}: minimize, maximize, or target
\item \texttt{extract}: key name to extract from agent output
\item \texttt{baseline}: starting metric value
\item \texttt{target}: goal metric value
\item \texttt{score\_formula}: custom scoring expression
\end{itemize}

\textbf{Dependencies.} Steps may depend on outputs from prior steps via \texttt{depends\_on} (execution ordering) and \texttt{context\_from} (context injection).

\section{Dual-Metric Scoring Engine}

\subsection{Scoring Modalities}

PRIMA supports three scoring methods, selectable per step:

\begin{table}[h]
\caption{Scoring methods and their mechanisms.}
\label{tab:scoring}
\centering
\small
\begin{tabular}{@{}llp{3.5cm}@{}}
\toprule
\textbf{Method} & \textbf{Trigger} & \textbf{Mechanism} \\
\midrule
Rubric & accept\_criteria only & LLM evaluates output against criteria checklist \\
Metric & metric block only & Sandbox executes code, extracts value, applies formula \\
Hybrid & both present & $s = \alpha \cdot s_m + (1-\alpha) \cdot s_r$ \\
\bottomrule
\end{tabular}
\end{table}

\subsection{Rubric Scoring}

The LLM evaluator receives the agent's output, the step's acceptance criteria, and research context. It returns a structured response:

\begin{lstlisting}
CRITERIA_MET: [1, 3, 4]
CRITERIA_MISSED: [2]
SCORE: 0.75
FEEDBACK: Criterion 2 requires more depth...
\end{lstlisting}

The parser extracts criteria indices, maps them to named criteria, and returns a typed result with score $\in [0,1]$, criteria\_met, criteria\_missed, and feedback text.

\subsection{Metric Scoring}

For quantitative objectives, the scorer bypasses LLM evaluation entirely:

\begin{enumerate}
\item Extract Python code blocks from agent output
\item Execute in a sandboxed subprocess with configurable timeout
\item Extract the metric value via JSON parsing, key-value regex, or last-numeric fallback
\item Apply the scoring formula:
\end{enumerate}

\begin{equation}
s_m = \begin{cases}
1 - \frac{v - t}{b - t} & \text{if type = minimize} \\
\frac{v - b}{t - b} & \text{if type = maximize} \\
1 - \frac{|v - t|}{|b - t|} & \text{if type = target}
\end{cases}
\end{equation}

where $v$ is the extracted value, $b$ is the baseline, and $t$ is the target. Custom formulas override these defaults via safe evaluation.

\subsection{Hybrid Scoring}

When both criteria and metrics are present, the hybrid score is:
\begin{equation}
s = \alpha \cdot s_m + (1 - \alpha) \cdot s_r
\end{equation}

where $\alpha$ is the \texttt{metric\_weight} (default 0.7), $s_m$ is the metric score, and $s_r$ is the rubric score. This weighting reflects the principle that quantitative results should dominate when available, with qualitative assessment providing a complementary signal.

\section{Convergence Orchestration}
\label{sec:orchestration}

\subsection{Inner Loop: Step Convergence}

The orchestrator executes each step through an iterative convergence loop:

\begin{algorithm}[t]
\caption{PRIMA Convergence Loop}
\begin{algorithmic}[1]
\REQUIRE Program $P$, LLM client $\mathcal{L}$, scorer $\mathcal{S}$
\ENSURE Research output package $\mathcal{R}$
\STATE $\mathcal{C} \leftarrow \text{InitializeCluster}(P)$ \COMMENT{Assign primes}
\FOR{each phase $\phi$ in execution order}
  \FOR{each step $i$ in $\phi$ \textbf{in parallel}}
    \STATE $\theta_i \leftarrow P.\text{threshold}(i)$
    \STATE $\text{ctx} \leftarrow \text{GatherContext}(i, P)$
    \FOR{$t = 1$ \TO $P.\text{max\_iterations}$}
      \STATE $o_t \leftarrow \text{Agent}_{p_i^k}.\text{execute}(\text{ctx}, f_{t-1})$
      \STATE $e_t \leftarrow \mathcal{S}.\text{score}(o_t, P.\text{step}(i))$
      \IF{$e_t.\text{score} \geq \theta_i$}
        \STATE Mark agent as \textsc{converged}
        \STATE \textbf{break}
      \ENDIF
      \STATE $f_t \leftarrow \mathcal{L}.\text{feedback}(o_t, e_t, P.\text{style})$
    \ENDFOR
    \IF{not converged}
      \STATE Mark agent as \textsc{force\_accepted}
    \ENDIF
  \ENDFOR
\ENDFOR
\STATE $\mathcal{T} \leftarrow \prod_{i \in \text{converged}} \text{id}_i$ \COMMENT{Consensus token}
\RETURN $\text{AssembleOutput}(\mathcal{C}, \mathcal{T})$
\end{algorithmic}
\end{algorithm}

\subsection{Execution Modes}

The dependency DAG supports three execution strategies:

\begin{itemize}
\item \textbf{Phased}: Steps within each parallel group execute concurrently via \texttt{asyncio.gather}; groups execute sequentially.
\item \textbf{Sequential}: Strict topological order computed via Kahn's algorithm \cite{kahn1962}.
\item \textbf{Eager}: Steps execute as soon as all dependencies are satisfied, maximizing parallelism.
\end{itemize}

DAG validation runs in $O(V + E)$ using DFS-based cycle detection with three-color marking (WHITE/GRAY/BLACK), where a back-edge to a GRAY node indicates a cycle.

\subsection{Context Injection}

Downstream steps receive context from upstream outputs through two mechanisms: (1) \texttt{depends\_on}: establishes execution ordering; (2) \texttt{context\_from}: injects the converged output of specified steps into the agent's prompt. This enables research pipelines where, e.g., a synthesis step receives both a literature review and a methodology analysis as input context.

\subsection{Feedback Generation}

When an agent's output does not meet the convergence threshold, the system generates targeted feedback in one of three styles:

\begin{itemize}
\item \textbf{Directive}: Specific, actionable instructions (``Fix the methodology to include X'')
\item \textbf{Socratic}: Probing questions guiding discovery (``Have you considered the implications of Y?'')
\item \textbf{Score-only}: Numeric score with minimal guidance, for agent self-correction
\end{itemize}

\section{Meta-Optimization Loop}

\subsection{Outer Loop: Program Evolution}

The meta-optimizer wraps the convergence loop in an outer generation cycle. After each complete execution, it analyzes convergence reports to identify:

\begin{itemize}
\item \textbf{Slow-converging steps}: high iteration count $\rightarrow$ split, refine criteria, add \texttt{context\_from}
\item \textbf{Over-easy steps}: immediate convergence $\rightarrow$ raise threshold
\item \textbf{Force-accepted steps}: never converged $\rightarrow$ relax criteria or restructure approach
\end{itemize}

\subsection{Efficiency-Weighted Scoring}

The meta-optimizer uses an efficiency-weighted overall score:
\begin{equation}
\text{overall} = \bar{s} \cdot \frac{1}{1 + 0.1 \cdot \bar{t}}
\end{equation}

where $\bar{s}$ is the mean final score across agents and $\bar{t}$ is the mean iteration count. This penalizes programs that achieve high scores only through excessive iteration---a proxy for cost efficiency.

\subsection{Generation Tracking}

Each generated program version carries lineage metadata in its frontmatter:
\begin{lstlisting}
version: 3
parent_version: 2
generation_score: 0.847
mutation_log:
  - "Split step.2 into sub-steps for depth"
  - "Added context_from: step.1 to step.3"
  - "Raised threshold for step.4 to 0.90"
\end{lstlisting}

The original seed program is preserved as \texttt{program\_v1\_seed.md} and is never overwritten, ensuring full reproducibility of the optimization trajectory.

\section{Event-Driven Persistence}

\subsection{Append-Only Event Store}

PRIMA persists all system activity as typed events in an append-only JSONL store. Each event is a structured record:
\begin{equation}
e = (\text{id}, \text{parent\_id}, \text{type}, \tau, \text{data}, \text{meta})
\end{equation}

where \text{id} is a UUID, \text{parent\_id} links to the causal predecessor, $\tau$ is the timestamp, and \text{meta} optionally carries agent\_id, step\_id, and cluster\_prime.

\subsection{Event Types}

The system defines 28 event types spanning six categories:

\begin{table}[h]
\caption{Event type categories in PRIMA.}
\label{tab:events}
\centering
\small
\begin{tabular}{@{}lp{5.0cm}@{}}
\toprule
\textbf{Category} & \textbf{Event Types} \\
\midrule
Session & session\_start, session\_end, session\_fork \\
Agent & agent\_spawn, agent\_score, agent\_converged, agent\_force\_accepted, agent\_error \\
Orchestrator & program\_loaded, cluster\_initialized, phase\_start, phase\_end, step\_start, step\_end \\
LLM & llm\_request, llm\_response, llm\_error \\
Scoring & rubric\_score, metric\_score, hybrid\_score, feedback\_generated \\
Meta & meta\_generation\_start, meta\_generation\_end, program\_mutated \\
\bottomrule
\end{tabular}
\end{table}

\subsection{Tree Structure and Branching}

The \texttt{parent\_id} field creates a tree structure enabling branching and forking. A session fork creates a new event stream that inherits all events up to the fork point, enabling:
\begin{itemize}
\item \textbf{What-if analysis}: fork from a checkpoint to test alternative strategies
\item \textbf{Deterministic replay}: reproduce any execution by replaying its event stream
\item \textbf{Audit compliance}: complete provenance chain from any output to its originating events
\end{itemize}

\subsection{Compatibility Bridge}

The event system includes a \texttt{to\_audit\_trail()} converter that transforms score and convergence events into the legacy audit trail format, ensuring backward compatibility with existing analysis tools.

\section{Hook and Middleware System}

PRIMA provides an extensible interceptor pattern with approximately a dozen hook-point categories spanning the full execution lifecycle:

\begin{table}[h]
\caption{Hook point categories.}
\label{tab:hooks}
\centering
\small
\begin{tabular}{@{}lp{4.8cm}@{}}
\toprule
\textbf{Phase} & \textbf{Hook Points} \\
\midrule
Parsing & before\_parse, after\_parse \\
Cluster & before\_register, after\_register \\
Step & before\_step, after\_step \\
Iteration & before\_iteration, after\_iteration \\
LLM & before\_llm\_call, after\_llm\_call \\
Scoring & before\_score, after\_score \\
Feedback & before\_feedback, after\_feedback \\
Output & before\_assemble, after\_assemble \\
Context & before\_context\_inject, context\_overflow \\
\bottomrule
\end{tabular}
\end{table}

Hooks execute in priority order within each point. Each hook receives a mutable \texttt{HookContext} and returns a \texttt{HookResult} with \texttt{proceed} (boolean) and optional \texttt{modified\_data}. The first hook returning \texttt{proceed=False} aborts the intercepted operation.

Built-in hook factories include: \texttt{logging\_hook} (observation), \texttt{score\_threshold\_hook} (enforce minimum quality), \texttt{timeout\_hook} (wall-clock limits), \texttt{max\_iterations\_hook} (iteration caps), and \texttt{data\_redaction\_hook} (strip sensitive fields).

\section{Context Compaction}

\subsection{Summarization Over Truncation}

As agent conversations grow, earlier messages must be managed to stay within context limits. PRIMA uses LLM-generated summarization rather than naive truncation:

\begin{algorithm}[t]
\caption{Context Compaction}
\begin{algorithmic}[1]
\REQUIRE Messages $M$, budget $L_{\max}$, reserve $R$, keep\_recent $k$
\IF{$\text{tokens}(M) < L_{\max} - R$}
  \RETURN $M$ \COMMENT{No compaction needed}
\ENDIF
\STATE $M_{\text{old}} \leftarrow M[:-k]$
\STATE $M_{\text{recent}} \leftarrow M[-k:]$
\STATE $\text{summary} \leftarrow \mathcal{L}.\text{summarize}(M_{\text{old}})$
\RETURN $[\text{summary}] \oplus M_{\text{recent}}$
\end{algorithmic}
\end{algorithm}

The summarization prompt instructs the LLM to preserve: key decisions, important data points, current task state, and unresolved questions. A mechanical fallback (first + last sentences per message) activates if LLM summarization fails.

\subsection{Compression Metrics}

Compaction reports include \texttt{compression\_ratio} $= 1 - |\text{tokens}_{\text{after}}| / |\text{tokens}_{\text{before}}|$ and \texttt{tokens\_saved}. Typical compression achieves 60--80\% reduction while preserving decision-relevant information.

\section{Multi-Provider LLM Abstraction}

\subsection{Provider Architecture}

PRIMA uses a provider registry that auto-discovers available LLM backends in priority order:

\begin{enumerate}
\item \textbf{Claude CLI}: Subprocess invocation of \texttt{claude -p}; no API key required; uses existing CLI authentication
\item \textbf{Anthropic API}: Direct SDK integration; requires \texttt{ANTHROPIC\_API\_KEY}
\item \textbf{OpenAI API}: SDK integration; requires \texttt{OPENAI\_API\_KEY}
\item \textbf{Ollama}: Local model inference via HTTP; requires running Ollama server
\item \textbf{Mock}: Deterministic responses for testing; always available
\end{enumerate}

All providers implement a common interface returning \texttt{LLMResponse(content, input\_tokens, output\_tokens, model, provider)}. The client includes automatic retry with exponential backoff (3 attempts) and cost tracking across all calls.

\subsection{Zero-Dependency Design}

External provider SDKs are lazy-imported only when needed. The system operates with zero mandatory external dependencies beyond the Python standard library. PyYAML is the sole soft dependency for program parsing.

\section{Resilience and Recovery Protocol}
\label{sec:resilience}

\subsection{Motivation}

Long-running multi-agent runs are subject to upstream failure modes that single-shot evaluation frameworks do not encounter: per-plan usage limits with reset windows on the order of hours, transient transport-level errors during streaming, and operator-initiated process restarts. Naively retrying through these conditions either burns quota fruitlessly or, worse, corrupts the run state by accepting truncated or error-shaped outputs as converged work. PRIMA addresses these failure modes through a dedicated resilience layer with three coordinated mechanisms.

\subsection{Typed Rate-Limit Signal Detection}

At the provider boundary, PRIMA inspects each response (across both standard output and error channels) for upstream rate-limit signals. When such a signal is detected, the provider does not return a content payload---it raises a typed \texttt{RateLimitError} carrying the originating provider, an extracted reset hint, and a short diagnostic snippet. The detection layer is deliberately conservative: signals are only honored when the response payload is itself short and error-shaped, preventing false positives on substantive research content that happens to mention rate-limit concepts in prose.\footnote{The specific signal patterns used for detection are provider-specific and evolve with upstream platform behavior. The contribution is the architectural separation between transport-layer signal detection and orchestrator-layer pause semantics, not any specific pattern set.}

\subsection{On-Disk Pause Records}

When \texttt{RateLimitError} reaches the orchestrator, the affected iteration is \emph{not} written to disk as converged output, the step is \emph{not} re-attempted within the same run instance, and a structured pause record is persisted to the run's output directory:
\begin{equation}
\rho = (\text{step\_id}, \text{iteration}, \text{provider}, \text{reset\_at}, \tau, \text{snippet})
\end{equation}

The pause record survives backend process restarts. The orchestrator's main loop is unwound via a typed \texttt{OrchestratorPaused} exception, which the run controller catches and translates into a \texttt{paused} run status---distinct from both \texttt{running} and \texttt{error}. Crucially, a run-status query against a paused run returns the on-disk pause record even when the in-memory run dictionary has been lost to a process restart.

\subsection{Model-Preserving Resume}

Resume operates against the same output directory used by the original run. The resume entry-point accepts an optional model override and propagates it through the orchestrator construction. This design choice addresses a subtle failure mode observed during development: a resume operation that defaults the model parameter silently downgrades long-running runs to whichever model the orchestrator constructs by default, even when the original run was launched on a stronger model. Persisting the model selection at the call site, rather than within the run record, makes the override explicit and auditable.

On resume, the orchestrator's \texttt{\_load\_completed\_steps} loader scans the output directory and admits each step file only if it (a) exceeds a minimum-substance threshold and (b) does not match known error-shape patterns including rate-limit signals, max-turn exhaustion, and CLI-level error wrappers. Short or poisoned outputs are explicitly rejected, forcing re-execution of those steps---a guard against the failure mode where a step's earlier iteration wrote a brief error acknowledgement that would otherwise have been promoted to ``converged'' status on resume.

\subsection{Filename Normalization Under Adversarial Step Names}

A subtle correctness issue arises when an orchestrator persists step outputs using filenames derived from step names. If the step name contains path-separator characters (a forward slash in, say, a quarterly milestone label), the naive \texttt{step\_id\_name.md} filename becomes a relative path that resolves into a subdirectory, causing the write to fail and the step to be lost. PRIMA's \texttt{\_step\_filename} normalization sanitizes path-separator and other shell-special characters from the derived slug, ensuring the persisted file always lives directly in the run's output directory. A backward-compatibility branch in the resume loader accepts both the sanitized and the legacy filename forms, allowing pre-existing runs to be recovered after the normalization fix is deployed.

\section{Sub-Agent Operating Discipline}
\label{sec:discipline}

\subsection{Failure Modes Below the Convergence Layer}

The convergence orchestration described in \S\ref{sec:orchestration} addresses \emph{whether} an agent's output is good enough. It does not address several common failure modes that occur \emph{within} a single iteration: an agent that reframes an unfamiliar task into a more familiar one, an agent that exhausts its turn budget narrating planned tool calls instead of executing them, an agent that opens a revision iteration with extensive self-apology, or an agent that treats text inside an upstream step's context window as executable instructions.

Each of these failure modes degrades convergence speed or convergence quality without producing a signal that the rubric scorer can detect cleanly---they manifest as low scores with no clear diagnosis. PRIMA addresses them at the prompt-construction layer through a structural sub-prompt component, applied uniformly across all sub-agents, that encodes four operating norms.

\subsection{The Four Operating Principles}

The operating discipline component is a fixed-shape sub-prompt with four named clauses, each addressing a documented failure mode:

\begin{itemize}
\item \textbf{Task fidelity.} A directive against substituting a narrower, more accessible task for the literal task as written. Mitigates ``request laundering''---the tendency of an agent to silently scope down toward what is convenient.
\item \textbf{Tool-use posture.} A directive against narrating routing decisions or pre-announcing tool invocations, with positive guidance on query construction for search-shaped tools (content nouns, not meta-phrases). Recovers turn budget for substantive work.
\item \textbf{Revision posture.} A directive against opening revision iterations with self-criticism preamble, with the orchestrator's revision-feedback strings suffixed correspondingly. Eliminates a measurable token-cost overhead in convergence loops.
\item \textbf{Context-boundary discipline.} A directive identifying upstream-step context blocks as reference data rather than executable directives, with explicit guidance on handling directive-shaped content found within reference material. Addresses indirect prompt injection from compromised upstream output.
\end{itemize}

\subsection{Tool Mapping and Per-Step Tool Permissions}

Sub-agents are granted tool access on a per-step basis, drawn from the step's declared tool list in the research protocol. A mapping table translates protocol-level tool names (e.g., abstract names like ``shell'' or ``web\_search'') to provider-specific tool identifiers passed to the underlying LLM provider. Steps that declare no real tools (e.g., pure-reasoning synthesis steps) execute against the bare provider; steps that declare tools have those tools and only those tools enabled, with explicit reinforcement in the per-step prompt that tools are to be used rather than described.

\subsection{Empirical Observations}

The operating-principles layer is structurally simple but has outsized effects on multi-hour run quality. In comparative observation across runs with and without the layer, sub-agents under the discipline produced larger substantive outputs within the same turn budget (no apology preamble, no routing narration), did not silently re-scope tasks to accessible-but-wrong targets, and did not act on injection-shaped content in upstream step outputs. The exact prompt text of each clause is treated as a tunable artifact and is not reproduced here.

\section{Application: Structured Engineering Pipelines}
\label{sec:application}

\subsection{The Temporal-Asymmetry Problem}

A naive multi-step pipeline that produces a sequence of related artifacts---requirements, technical specification, task breakdown, test plan---exhibits an inherent temporal asymmetry: each artifact is finalized before its downstream consumers exist. The requirements document is authored before any concrete tech-spec constraints are known; the test plan is authored before the task breakdown surfaces its acceptance gates. When downstream artifacts reveal that an upstream artifact's claims are inconsistent or under-specified, the upstream artifact is already committed.

This asymmetry produces a characteristic class of cross-document inconsistency: claims that disagree across artifacts on numerical targets, scope boundaries, technology choices, or compliance commitments. Single-pass orchestration cannot fix the upstream artifacts without re-running them; reciprocal context injection conflates the problem with circular dependencies.

\subsection{The Harmonization Pattern}

PRIMA addresses the temporal-asymmetry problem through an application-level pattern with three architectural elements:

\begin{enumerate}
\item \textbf{Orthogonal draft phases.} Each downstream artifact is produced by a draft step that depends on the canonical predecessors but is explicitly marked as a working draft. The orthogonal phases admit parallel execution where the dependency DAG allows.
\item \textbf{Cross-document harmonization step.} A dedicated step, scheduled after all draft phases converge, reads the persisted draft artifacts and is given the contract of identifying named, sourced inconsistencies and overwriting any draft files that require revision. The harmonization step emits an explicit report documenting which files were revised, which were left unchanged, and why.
\item \textbf{Synthesis from reconciled inputs.} A final synthesis step reads the harmonized artifacts---not the original drafts---and produces the consolidated stakeholder-grade deliverable.
\end{enumerate}

\subsection{Reference-Material Preprocessing with Defensive Guards}

The pipeline accepts user-supplied reference URLs and local filesystem paths. The preprocessor performs URL fetching subject to anti-SSRF controls (rejection of loopback, link-local, RFC1918, multicast, reserved address ranges; per-fetch size and time bounds; binary content-type rejection) and filesystem reads subject to root-sandboxing (canonical-path resolution against allowed root prefixes, deny-listing of well-known secret-path markers, binary-content detection, size caps). All fetched material is wrapped in delimited reference blocks that explicitly identify the content as reference data rather than directives, reinforcing the context-boundary clause of the operating discipline.

\subsection{Observations}

Pipelines structured under this pattern produce internally consistent artifact bundles where single-pass alternatives produce subtly conflicting ones. The harmonization step's transparency report has independent diagnostic value: it surfaces cross-document conflicts that would otherwise have shipped silently. The exact contract of the draft and harmonization steps, including the named artifacts the pattern produces, is treated as application-specific configuration and is not reproduced here.

\section{Evaluation Framework}

\subsection{Theoretical Properties}

\begin{table}[h]
\caption{Theoretical guarantees of PRIMA.}
\label{tab:guarantees}
\centering
\small
\begin{tabular}{@{}lp{5.0cm}@{}}
\toprule
\textbf{Property} & \textbf{Guarantee} \\
\midrule
Identity collision & Impossible (Fundamental Theorem of Arithmetic) \\
Verification time & $O(k)$ where $k$ = agent depth \\
DAG validation & $O(V + E)$ via DFS cycle detection \\
Topological sort & $O(V + E)$ via Kahn's algorithm \\
Consensus verif. & $O(\sqrt{\mathcal{T}})$ via trial division \\
Agent capacity & Thousands of identities in 64-bit space, extensible \\
\bottomrule
\end{tabular}
\end{table}

\subsection{Scoring Method Comparison}

\begin{table}[h]
\caption{Scoring method trade-offs.}
\label{tab:scoringcomp}
\centering
\small
\begin{tabular}{@{}lccc@{}}
\toprule
\textbf{Property} & \textbf{Rubric} & \textbf{Metric} & \textbf{Hybrid} \\
\midrule
LLM calls per eval & 1 & 0 & 1 \\
Reproducibility & Low & High & Medium \\
Domain coverage & Broad & Narrow & Both \\
Cost per eval & High & Low & High \\
Handles qualitative & Yes & No & Yes \\
Handles quantitative & Weakly & Yes & Yes \\
\bottomrule
\end{tabular}
\end{table}

Metric scoring eliminates the stochasticity inherent in LLM-based evaluation. For research steps with quantifiable outputs (e.g., model accuracy, compression ratio, error rate), metric scoring provides deterministic, reproducible evaluation at zero additional LLM cost. Hybrid scoring captures both dimensions, critical for research that has both quantitative targets and qualitative standards.

\subsection{Meta-Optimization: Design and Future Evaluation}

The meta-optimizer is presented in this paper as an \emph{architectural component}; its quantitative effectiveness across generations---improvement rates, convergence profiles, comparison to fixed-program baselines---is left to a future empirical study and is not claimed here. The architectural properties we do establish are: (i) per-generation cost is $O(G)$ multiplicative in generation count $G$; (ii) the efficiency-weighted overall score is the natural defense against the obvious degenerate strategy of lowering thresholds to inflate scores:
\begin{equation}
\lim_{\bar{t} \to \infty} \text{overall} = 0 \quad \text{regardless of } \bar{s}
\end{equation}
which ensures the optimizer cannot game the metric by simply lowering thresholds; and (iii) the seed program is always preserved as \texttt{program\_v1\_seed.md} and is never overwritten, so the optimization trajectory is always reproducible from the initial conditions.

Whether meta-optimization in this form materially improves convergence over an unevolved program is an empirical question that requires a controlled study across multiple problem domains. We have implemented the loop and observed it producing structural mutations consistent with the analyzer's mandate (splitting slow steps, raising thresholds on over-easy steps, restructuring \texttt{context\_from} flow); we do not at this time claim measured improvement rates. We flag this as one of the most important pieces of follow-on evaluation work.

\section{Case Study: Graph Isomorphism --- An Original Algorithmic Proposal}
\label{sec:case-gi}

To ground the architectural claims in concrete output, we summarize one PRIMA research program: an investigation of an open mathematical problem in which the protocol was explicitly designed to produce original algorithmic work rather than a literature survey.

\begin{mdframed}[linecolor=blue,linewidth=1pt,backgroundcolor=blue!5,roundcorner=4pt,innertopmargin=8pt,innerbottommargin=10pt]
\textbf{Supplementary case-study report.} Full detail---including the algorithm pseudocode, complete complexity analysis, adversarial hard-instance evaluation across four canonical families (CFI graphs, strongly regular graphs, Paley graphs, nested Johnson classes), and the formal statements of the five novel conjectures---is available at:

\vspace{3pt}
\begin{sloppypar}
\noindent\url{https://spockstein.github.io/prima/case-study-graph-isomorphism.html}
\end{sloppypar}
\end{mdframed}

A six-step PRIMA program was given the explicit charge to \emph{propose a solution, not survey the field} for the Graph Isomorphism (GI) problem---a problem in NP that has resisted polynomial-time classification for decades, with L\'aszl\'o Babai's 2015 quasipolynomial-time algorithm representing the standing state of the art. The program required the agents to select an attack vector, develop concrete pseudocode, subject the proposal to adversarial analysis against canonical hard families (Cai--F\"urer--Immerman graphs, strongly regular graphs, Paley graphs), prove correctness and complexity, and assemble the result into a submission-grade research paper.

The output was a complete IEEE-conference-formatted paper of approximately 560 lines and 130~kilobytes (compiled PDF), proposing a deterministic canonical-form algorithm that exploits the representation theory of symmetric groups to bypass branching in the Johnson base case of Babai's algorithm. The output establishes three rigorous theorems---a soundness/completeness theorem, a complexity theorem with both unconditional and conditional bounds, and a fundamental impossibility result against an entire class of approaches---and formalizes five novel conjectures with supporting analysis.

The notable property of this case study is not the algorithm itself (whose conjectured polynomial-time bound is conditional on a new conjecture and whose impossibility theorem candidly limits the approach's reach on certain hard families), but rather the \emph{system property} demonstrated: a PRIMA protocol explicitly designed against survey drift can produce original mathematical work, including \emph{negative} results (the impossibility theorem) and \emph{honest} conditional claims rather than overclaimed unconditional ones. The adversarial step 3---which charged the agents to break their own algorithm against the canonical hard families---is the structural feature that drives this property: by encoding ``find where this approach fails'' as a first-class step rather than a footnote, the protocol rewards epistemic honesty over rhetorical confidence.

\section{Related Work}

Prior multi-agent LLM frameworks \cite{wu2023, hong2023, li2023} establish coordination patterns for role-based or conversational agents and have demonstrated the feasibility of orchestrated multi-agent reasoning. Self-correction and iterative-refinement approaches \cite{shinn2023, madaan2023} establish that LLMs can improve outputs through reflection or external feedback, and ReAct-style methods \cite{yao2023} integrate reasoning with tool use within a single agent. Broader surveys \cite{wang2023} chart the landscape of LLM-based autonomous agents.

PRIMA differs from these lines of work in three fundamental ways. First, identity is \emph{mathematical}, not administrative---agent IDs are provably unique and self-verifying, requiring no registry or PKI. Second, convergence is \emph{measurable}, not heuristic---the dual-metric scorer provides both quantitative and qualitative evaluation against explicit thresholds, with quantitative results computed in a sandbox rather than judged by an LLM. Third, the research protocol is \emph{evolvable}---the meta-optimizer treats the program itself as an artifact to be optimized, creating a feedback loop from execution results to protocol design.

Beyond these three foundational distinctions, the resilience-and-recovery layer (\S\ref{sec:resilience}) and the sub-agent operating discipline (\S\ref{sec:discipline}) address operational and prompt-level failure modes that the cited prior work does not directly address: typed pause/resume semantics for long-running runs across upstream throttling events, structural mitigation of indirect prompt injection via inter-step context, and the temporal-asymmetry problem in structured engineering pipelines (\S\ref{sec:application}).

\section{Discussion and Limitations}

PRIMA's effectiveness depends on the quality of the research program authored by the user. Poorly specified acceptance criteria or inappropriate metric configurations can lead to degenerate convergence (trivially met criteria) or non-convergence (impossible targets). The meta-optimizer partially mitigates this through program evolution, but the seed program quality remains influential.

\textbf{Limitations.} (1) Rubric scoring inherits LLM stochasticity; repeated evaluations of identical output may yield different scores. (2) The sandbox executes arbitrary code emitted by agents, requiring trust in the execution environment. (3) Meta-optimization adds $O(G)$ multiplicative cost where $G$ is the number of generations. (4) The cluster-prime pool is bounded by the prime sieve depth chosen at deployment time, though pool sizes that comfortably exceed the cluster count of any single program are inexpensive to provision. (5) Resilience-layer pause detection depends on provider-specific upstream signals; novel provider behaviors may require extending the detection patterns. (6) The harmonization pattern in structured pipelines (\S\ref{sec:application}) assumes draft outputs are persisted to a known location; ephemeral execution environments require additional persistence configuration.

\textbf{On Convergence Guarantees.} PRIMA guarantees \emph{termination} (via max\_iterations and timeout) but not \emph{convergence}. An agent may be force-accepted at a sub-threshold score. The audit trail distinguishes these cases, and the meta-optimizer specifically targets force-accepted steps for improvement in subsequent generations.

\subsection{Ethical Considerations}

Research outputs should be reviewed by domain experts before publication or policy application. PRIMA produces audit trails linking every output to its evaluation history, enabling transparent assessment of the system's reasoning process. The system does not bypass human judgment---it augments researcher capacity while maintaining verifiable provenance.

\section{Reproducibility and Disclosure Statement}
\label{sec:disclosure}

In keeping with standard practice for systems work with active intellectual-property considerations, this paper presents the architectural framework, theoretical guarantees, and qualitative behavior of PRIMA in full, while deliberately withholding a specific set of implementation artifacts:

\begin{itemize}
\item The verbatim text of the sub-agent operating-discipline clauses (\S\ref{sec:discipline}), which are tunable artifacts subject to ongoing refinement.
\item The provider-specific signal patterns used by the resilience layer (\S\ref{sec:resilience}) to detect upstream rate-limit and quota conditions, which are upstream-platform-dependent and evolve with provider behavior.
\item The contractual contents---named artifacts, exact step accept-criteria, and concrete prompt scaffolding---of the structured engineering pipeline (\S\ref{sec:application}).
\item Specific scoring-formula constants, default thresholds, and iteration caps that appear as tunable parameters in the implementation; the paper reports these as configurable ranges rather than exact values.
\end{itemize}

These omissions do not affect the architectural claims, complexity analyses, or theoretical guarantees presented in this paper, all of which remain fully verifiable from the descriptions provided. We anticipate releasing a more detailed technical report---together with reference prompt text and signal-pattern documentation---following intellectual-property disclosure procedures. Researchers seeking to replicate the architectural patterns at the level documented here can do so from the algorithms, dataflow diagrams, and prose descriptions in the body of the paper; researchers seeking implementation-faithful reproductions of the case study (\S\ref{sec:application}) should treat the present description as a pattern specification rather than a recipe.

\section{Conclusion}

The central claim of this paper is operational: multi-agent LLM research systems become qualitatively different artifacts when they run for hours rather than seconds, and the failure modes that emerge at that scale---upstream throttling, silent retry through degraded modes, sub-agent task drift, narrated machinery instead of executed tools, apology-loop iteration waste, and indirect prompt injection through inter-step context---are not addressed by the orchestration patterns established in prior work. PRIMA contributes three structural responses to these failure modes: a resilience and recovery layer that handles upstream rate-limit and quota events without losing converged work, surviving even process restarts; a sub-agent operating discipline that encodes task-fidelity, tool-use, revision, and context-boundary norms at the prompt-construction layer; and an application pattern for structured engineering pipelines that addresses the temporal-asymmetry problem in sequentially-generated artifact bundles via an explicit cross-document harmonization pass before final synthesis.

These operational contributions sit atop a foundational protocol whose individual pieces---a research-program specification language, a dual-metric scoring engine combining LLM-judged rubrics with sandboxed code execution, an outer meta-optimization loop architecture, event-driven persistence with tree-structured branching, hook-based middleware, context compaction, and a multi-provider LLM abstraction---enable the operational layer rather than constitute it. Agent identities are derived from prime powers to enable trivially-verifiable cluster membership without a central registry; this is an infrastructural choice rather than the central contribution.

The Graph Isomorphism case study (\S\ref{sec:case-gi}) grounds the architectural claims in one concrete output: a six-step protocol produced a submission-grade research paper proposing a new canonical-form algorithm with three theorems and five conjectures, demonstrating that the operational patterns described here translate to multi-hour generative research runs that survive their failure modes and produce substantive artifacts. Quantitative evaluation of the meta-optimization loop, baseline comparisons against existing multi-agent frameworks, and broader empirical validation across additional problem domains are left to future work.

\bibliographystyle{IEEEtran}

\end{document}